\documentclass[]{ustc_conference}
\usepackage{amssymb}
\usepackage[utf8]{inputenc}
\usepackage[ruled,vlined]{algorithm2e}
\usepackage{booktabs}
\usepackage{multirow}
\usepackage{xcolor}
\usepackage{adjustbox}
\usepackage[edges]{forest}
\usepackage{tikz}
\usepackage{amsfonts}
\usepackage[toc,page,header]{appendix}
\usepackage{bbm}
\usepackage{amsthm}
\usepackage{enumitem}
\usepackage{mathtools}
\usepackage{wrapfig}
\newtheorem{theorem}{Theorem}
\renewcommand{\paragraph}[1]{\vspace{0.1em}\noindent\textbf{#1}}

\title{SocraticPO: Policy Optimization via Interactive Guidance}

\author{%
\parbox{\textwidth}{\centering
Zirui Liu\frontsup{1}, Tingyue Pan\frontsup{1}, Jie Ouyang\frontsup{1}, Qi Liu\frontsup{1\corrauthor}, Xianquan Wang\frontsup{1}, Jiayu Liu\frontsup{1}, Qingchuan Li\frontsup{1}, Jing Sha\frontsup{2}, Zhenya Huang\frontsup{1}, Shijin Wang\frontsup{2}, Enhong Chen\frontsup{1}
}}

\affiliation{%
\parbox{\textwidth}{\centering
\affilsup{1}State Key Laboratory of Cognitive Intelligence, University of Science and Technology of China\\
\affilsup{2}iFLYTEK AI Research (Central China), iFLYTEK Co., Ltd
}
}

\correspondence{\email{qiliuql@ustc.edu.cn}}

\abstract{
Reinforcement learning (RL) for large language models usually supervises reasoning with scalar outcome rewards, such as binary correctness. Such rewards provide an optimization direction but rarely explain how a model should revise its mistaken reasoning, which can encourage shortcut learning and brittle policies. We propose \textbf{SocraticPO} (Socratic Policy Optimization), a policy-optimization framework that augments RL rollouts with Socratic-style natural-language guidance. During rollout, the student first answers independently; if the answer is incorrect, a teacher diagnoses the attempt and provides concise corrective guidance, after which the student continues under the expanded context. Crucially, this guidance is paired with reward decay: correct answers obtained after teacher intervention only receive decayed rewards, preventing the policy from treating teacher help as a free path to reward. Since SocraticPO only modifies the rollout process while leaving the standard expected-reward objective intact, it can be plugged into existing policy-gradient backends such as Reinforce++. Moreover, because the teacher provides only text-level guidance, SocraticPO can leverage stronger black-box teacher models without requiring access to logits or distribution matching. On undergraduate-level scientific reasoning benchmarks from SciKnowEval, SocraticPO improves over strong RL and self-distillation baselines. Ablations show that both targeted guidance and reward decay are necessary, with reward decay mitigating reliance on assisted correction.
}

\checkdata[SocraticPO]{\url{https://github.com/Liuz-rui/SocraticPO}}

\begin{document}

\maketitle

\begin{figure*}[t]
\centering
\includegraphics[width=0.8\textwidth]{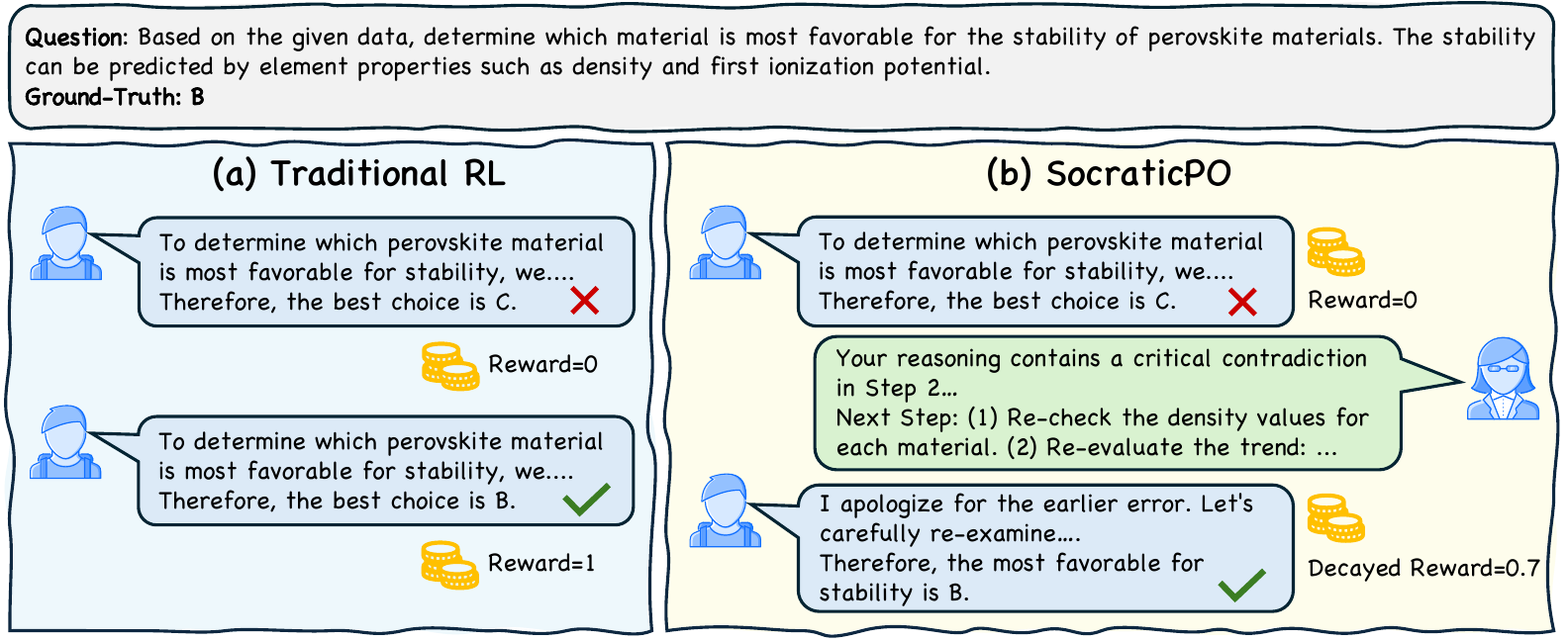}
\caption{Overview of SocraticPO. Conventional RL rollouts provide only scalar outcome feedback after a student response, whereas SocraticPO inserts conditional teacher guidance after failed attempts.}
\label{fig:mdp-formulation}
\end{figure*}

\section{Introduction}

Large language models (LLMs) have advanced rapidly with transformer architectures, scaling, and instruction tuning~\cite{vaswani2017attention,brown2020language,touvron2023llama,achiam2023gpt,bai2023qwen}. Their post-training increasingly relies on reinforcement learning (RL), where a policy is optimized from task rewards or preference-derived signals~\cite{christiano2017deep,stiennon2020learning,ouyang2022training}. For reasoning tasks, however, RL feedback is often reduced to a scalar outcome signal, such as binary correctness. Such feedback tells the model which trajectories should be reinforced, but not how a mistaken reasoning process should be repaired~\cite{cai2025reinforcement}.

This limitation is especially problematic in multi-step scientific reasoning. A wrong final answer may arise from a localized conceptual error, an incorrect comparison, or a misplaced intermediate value. Outcome rewards provide gradient-level optimization pressure, but they do not provide thought-level correction. As a result, the policy may discover shortcuts that satisfy the verifier without learning robust reasoning.

Human learning offers a different pattern of supervision. Teachers rarely provide only a grade; they inspect the learner's attempt, identify the source of the mistake, and give hints, critiques, or next-step suggestions~\cite{wood1976role,vanlehn2011relative}. The Socratic tradition frames this process as guided inquiry rather than direct answer transfer~\cite{scott2006plato,liu2024socraticlm}. Inspired by this view, we introduce \textbf{Socratic Policy Optimization (SocraticPO)}, which brings Socratic-style guidance into the rollout stage of RL. A student first attempts the problem independently. If the attempt is incorrect, a teacher provides natural-language guidance conditioned on the student's response, and the student continues under the expanded context.

Guidance must be introduced carefully. If a response corrected after teacher intervention receives the same reward as one solved independently, the policy may learn a training-time shortcut: wait for help, use the teacher-provided clue, and still collect full reward. Assisted correctness can then improve without a corresponding improvement in unassisted reasoning. SocraticPO therefore pairs teacher guidance with \textbf{reward decay}: correct answers obtained after teacher intervention receive decayed rewards, and the reward decreases as more guidance rounds are needed. This encourages the student not only to reach the correct answer, but also to internalize the correction and solve earlier with less assistance.

SocraticPO changes only the rollout process. The student still optimizes a standard expected-reward objective, while teacher utterances are appended as text context for subsequent student attempts rather than optimized as imitation targets. As a result, SocraticPO can be instantiated with existing policy-gradient backends such as Reinforce++~\cite{hu2025reinforce}. Since the teacher is required only to produce text-level guidance, the framework can also leverage stronger black-box teacher models without accessing logits, token-level likelihoods, or matching teacher and student distributions, unlike probability-matching approaches such as OPD~\cite{agarwal2024policy,fu2026revisiting}.

We evaluate SocraticPO on undergraduate-level scientific reasoning benchmarks derived from SciKnowEval~\cite{feng2024sciknoweval}. Compared with strong RL baselines, including Reinforce++~\cite{hu2025reinforce} and GRPO~\cite{shao2024deepseekmath}, as well as self-distillation-style training such as SDPO~\cite{hubotter2026reinforcement}, SocraticPO achieves the best average performance. Ablations further show that both components are necessary: targeted guidance provides useful correction beyond blind retrying, while reward decay mitigates reliance on teacher-assisted correction.

Our main contributions are:
\begin{itemize}[leftmargin=*,itemsep=0.15em,parsep=0pt,topsep=0.2em]
    \item \textbf{A general SocraticPO framework.} We propose an interactive policy-optimization framework that augments RL with teacher-guided reasoning correction and remains compatible with existing policy-gradient backends.
    \item \textbf{Reward decay for SocraticPO.} To make teacher-guided RL train the student rather than reward dependence on assistance, we introduce a reward decay mechanism for assisted correctness. We further characterize its normalized-advantage properties, showing non-negativity, monotonicity, and saturation decay.
    \item \textbf{Black-box teacher compatibility.} Our method only requires teachers to provide text-level guidance instead of distribution matching, making it compatible with stronger API-based or otherwise black-box teachers.
    \item \textbf{Empirical evidence and analysis.} We show gains on scientific reasoning benchmarks and provide ablations on teacher guidance, reward decay, teacher-side information, and teacher--student similarity.
\end{itemize}

\section{Related Work}
\label{sec:related_work}

\subsection{Reinforcement Learning for LLMs}

Reinforcement learning has become a central paradigm for aligning LLMs with human preferences and task rewards~\cite{christiano2017deep,ziegler2019fine,stiennon2020learning,ouyang2022training,li2026reasoningtoolusecompeteagentic}. PPO~\cite{schulman2017proximal} is widely used in RLHF pipelines, and feedback can also be provided by AI systems rather than direct human annotation~\cite{bai2022constitutional}. More recent LLM training methods adopt simpler or more scalable policy-gradient variants, including REINFORCE-style updates~\cite{williams1992simple,sutton1999policy,hu2025reinforce} and group-wise relative objectives such as GRPO~\cite{shao2024deepseekmath}.

Most of these methods optimize scalar reward signals~\cite{guo2026treebased}. In reasoning tasks, such rewards identify whether a trajectory should be reinforced but do not explain which part of the reasoning process should be corrected~\cite{tan2025bottom,cai2026vi}. Process-supervision work shows that intermediate feedback can be more informative than outcome supervision alone~\cite{uesato2022solving}, yet how to inject corrective information into on-policy RL rollouts remains underexplored. SocraticPO addresses this gap by adding natural-language correction to rollout sampling while retaining a standard policy-gradient objective.

\subsection{Human Guided Teaching}

Educational research emphasizes that effective teaching is interactive rather than purely evaluative. Tutors scaffold learning by inspecting a student's current attempt, diagnosing misconceptions, and providing hints or next-step guidance~\cite{wood1976role,vanlehn2011relative}. The Socratic tradition similarly frames learning as guided inquiry: questions and critiques help the learner revise an answer without simply copying an expert solution~\cite{scott2006plato}. Recent NLP work has explored Socratic tutoring and personalized teaching dialogues with LLMs~\cite{liu2024socraticlm}.

SocraticPO adapts this teaching principle to RL training. The teacher does not merely label an answer as wrong; it provides language-level correction conditioned on the student's failed attempt. At the same time, SocraticPO differs from tutoring systems whose goal is to optimize dialogue quality. Our goal is to improve the student policy through RL, so teacher guidance is paired with a reward design that encourages the student to solve earlier and rely less on future intervention.

\subsection{Student--Teacher Paradigms in LLMs}

Teacher--student learning is widely used in LLM post-training to transfer capability from stronger models to smaller or weaker students. One common paradigm is teacher-generated supervised fine-tuning (SFT), where a stronger model produces instructions, answers, or reasoning traces that are then treated as demonstration data for the student. Self-Instruct, Alpaca, WizardLM, Orca, and step-by-step distillation all follow this broad pattern~\cite{wang2023selfinstruct,taori2023alpaca,xu2023wizardlm,mukherjee2023orca,hsieh2023distilling}. Recent work also studies how to synthesize student-consistent SFT data from a teacher so that the demonstrations better match the student's learning state~\cite{huang2026fine}.

Another paradigm is knowledge or policy distillation, where the student is trained to approximate a teacher's behavior more directly. Classical distillation matches teacher distributions or soft targets~\cite{hinton2015distilling}, while on-policy distillation adapts this idea to LLMs by using model-generated or teacher-shaped trajectories during training~\cite{agarwal2024policy,song2026survey,fu2026revisiting}. Self-distillation and hybrid distillation variants further use the model's own successful responses, privileged trajectories, or routed samples as training targets~\cite{hubotter2026reinforcement,li2026srpo,ding2026hdpo,ye2026opcd,tu2026ucob}. Despite differences in data source and objective, these methods generally ask the student to imitate, approximate, or absorb teacher-like behavior.

SocraticPO adopts a different teacher--student relation. The teacher is not a demonstration generator whose answer should be copied, nor a distributional target whose probabilities should be matched. Instead, it acts as an interactive guide: after a failed student attempt, the teacher diagnoses the mistake and provides concise corrective guidance, while the student still produces its own subsequent response and is optimized only through task reward. Because this guidance is text-level, SocraticPO can also use black-box teachers as an additional benefit of the design.

\section{Preliminaries}
\label{sec:preliminaries}

We first introduce the standard policy-gradient formulation for single-turn LLM training, then use it to clarify what SocraticPO changes: the rollout distribution rather than the expected-reward objective.

\subsection{Language Model Policy}

Given an input prompt $x$, an autoregressive language model with parameters $\theta$ defines a policy $\pi_\theta(y \mid x)$ over response sequences $y=(y_1,\dots,y_T)$. The response probability factorizes token by token:
\begin{equation}
\pi_\theta(y \mid x)=\prod_{t=1}^{T}\pi_\theta(y_t \mid x,y_{<t}).
\end{equation}
In the standard single-turn setting, one sampled response is treated as a trajectory $\tau=y$ and assigned a scalar reward $R(\tau)$, such as a correctness score.

\subsection{Expected-Reward Objective}

Policy optimization maximizes the expected reward under the current policy~\cite{williams1992simple,sutton1999policy}:
\begin{equation}
\mathcal{J}(\theta)=
\mathbb{E}_{\tau\sim\pi_\theta(\cdot\mid x)}[R(\tau)].
\end{equation}
For a sequence-level reward, the corresponding policy-gradient estimator can be written as:
\begin{equation}
\nabla_\theta \mathcal{J}(\theta)
=
\mathbb{E}_{\tau\sim\pi_\theta}
\left[
R(\tau)\sum_{t=1}^{T}
\nabla_\theta \log \pi_\theta(y_t\mid x,y_{<t})
\right].
\end{equation}
Practical implementations often add baselines, reward normalization, clipping, or group-wise comparison to reduce variance and stabilize training~\cite{schulman2017proximal,shao2024deepseekmath,hu2025reinforce}.

\subsection{Rollout Sampling}

The rollout distribution determines which trajectories the policy learns from. In conventional LLM RL, the rollout is sampled only from the student policy:
\begin{equation}
\tau \sim \pi_\theta(\cdot\mid x).
\end{equation}
The reward is then computed after generation, and the same trajectory is used for the policy update.

\subsection{Verifier and Correctness Notation}

For each prompt $x$, let $\mathcal G$ denote the information used by the verifier, such as the ground-truth answer and optional reference solution. The teacher may also receive side information $\mathcal I$, which can be empty, restricted, or include reference material depending on the teaching protocol. We define a correctness function
\begin{equation}
\delta(y,\mathcal G)=\mathbb{I}[\text{correct}(y,\mathcal G)].
\end{equation}



\section{Method}
\label{sec:method}

Traditional scalar rewards only tell the model whether an answer is correct or incorrect. They provide an optimization signal, but not a diagnosis of the reasoning error that caused the failure. In human learning, a teacher often plays this diagnostic role: by pointing out where the student's current attempt goes wrong, the teacher can help the student revise more efficiently than a grade alone. SocraticPO turns this idea into a policy-optimization mechanism. It modifies the rollout stage by inserting conditional teacher feedback after failed student attempts, so that the policy learns from trajectories that contain both the original mistake and a correction-oriented context for the next attempt.

\subsection{Interactive Rollout}

Using the notation from Section~\ref{sec:preliminaries}, the rollout is organized as a sequence of student attempts rather than a single long generation. This step-wise structure gives the verifier a chance to identify failure points and allows the teacher to respond to the student's actual mistake, instead of providing generic hints before the model has committed to a line of reasoning.
At each interaction step $k$ with $1\leq k\leq K$, the dialogue context is denoted by $c^{(k)}$, with $c^{(1)}=x$. The student then generates a complete response:
\begin{equation}
y^{(k)}=(y^{(k)}_1,\dots,y^{(k)}_{T_k})
\sim \pi_\theta(\cdot\mid c^{(k)}).
\end{equation}
Equivalently, the response probability factorizes at the token level as
\begin{equation}
\pi_\theta(y^{(k)}\mid c^{(k)})
=
\prod_{t=1}^{T_k}
\pi_\theta(y^{(k)}_t\mid c^{(k)},y^{(k)}_{<t}).
\end{equation}
The response is then evaluated by
\begin{equation}
\delta^{(k)}=\delta(y^{(k)},\mathcal G).
\end{equation}
If $\delta^{(k)}=1$, the rollout terminates. If $\delta^{(k)}=0$ and $k<K$, the teacher is queried with a teaching prompt that includes the current dialogue context, the student's latest response, and the teacher-side information $\mathcal I$. The teacher then generates concise natural-language feedback, such as an error diagnosis, a correction hint, or a suggested next step. We provide the prompt template in Appendix~\ref{app:prompt_templates}. Formally,
\begin{equation}
f^{(k)}\sim \pi_\phi(\cdot\mid c^{(k)},y^{(k)},\mathcal I),
\end{equation}
and the next context is constructed as
\begin{equation}
c^{(k+1)}=c^{(k)} || y^{(k)} || f^{(k)}.
\end{equation}
Thus teacher feedback appears only after an incorrect student response. A sampled interaction trajectory can be written as
\begin{equation}
\tau=
\left(
x,y^{(1)},f^{(1)},\dots,
f^{(L-1)},y^{(L)}
\right),
\end{equation}
where $L\le K$ is the number of student attempts actually made, and $f^{(k)}$ is present only when $\delta^{(k)}=0$ and $k<L$. This defines an interaction-conditioned rollout distribution
\begin{equation}
\tau\sim \pi_{\theta,\phi}^{\mathrm{int}}(\cdot\mid x,\mathcal G,\mathcal I,K).
\end{equation}
Only the student factors in this distribution depend on $\theta$; teacher outputs are sampled conditioning variables.

\subsection{Policy Optimization}

We keep the expected-reward objective but evaluate it under the interaction-conditioned rollout distribution. This preserves compatibility with standard policy-gradient training: SocraticPO changes which trajectories are sampled, not the principle that the student is reinforced by task reward.
\begin{equation}
\mathcal J(\theta)
=
\mathbb{E}_{\tau\sim \pi_{\theta,\phi}^{\mathrm{int}}}
\left[R(\tau)\right].
\end{equation}
In our step-wise rollout, each generated student response receives a reward $r^{(k)}$, and the trajectory return is
\begin{equation}
R(\tau)=\sum_{k=1}^{L} r^{(k)}.
\end{equation}
Using step-wise rewards is important for within-trajectory credit assignment\cite{wang2026steppo}. A final trajectory-level reward would mark the whole interaction as successful once the last answer is correct, but it would not identify which student attempt should receive credit. Step-wise rewards instead attach the learning signal to the specific response that becomes correct and leave earlier failed attempts unrewarded. This separates independent success, recovery after one targeted hint, and recovery after repeated intervention inside the same trajectory.
The policy-gradient update is applied only to student tokens:
\begin{equation}
\nabla_\theta \mathcal J(\theta)
\approx
\sum_{k=1}^{L}
r^{(k)}
\sum_{t=1}^{T_k}
\nabla_\theta
\log \pi_\theta(y^{(k)}_t\mid c^{(k)},y^{(k)}_{<t}).
\end{equation}

\subsection{Reward Decay}
\begin{figure*}[t]
\centering
\includegraphics[width=\textwidth]{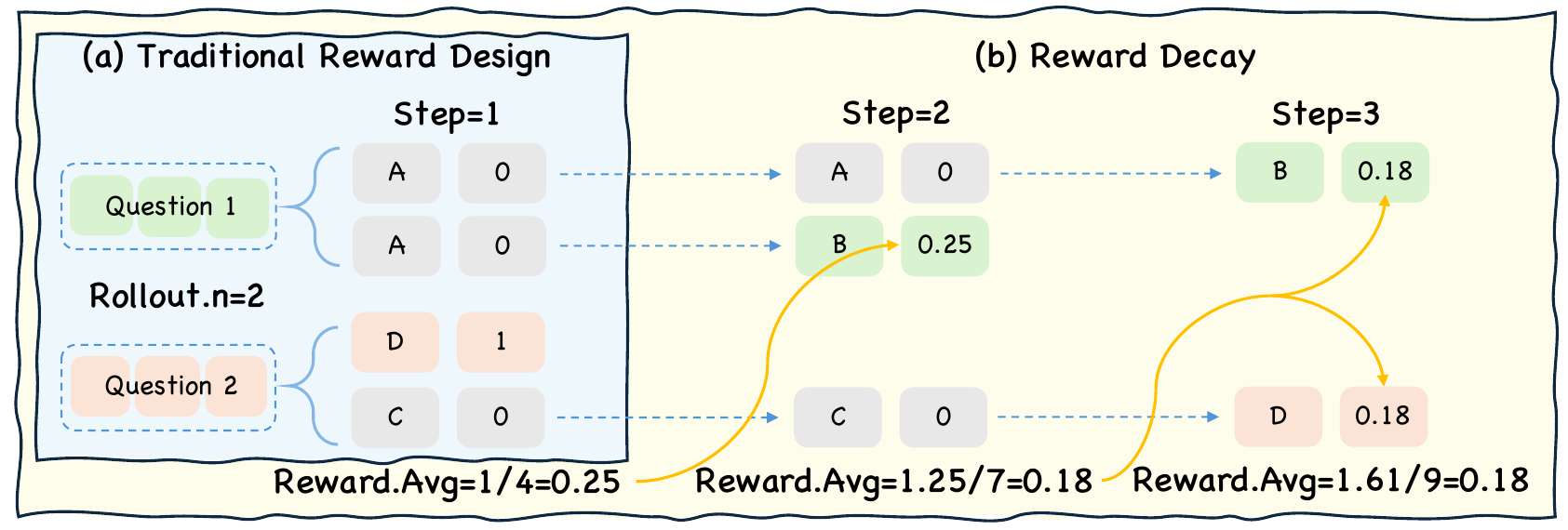}
\caption{Step-wise reward and reward decay in SocraticPO. Step-wise rewards assign credit to the specific student attempt that becomes correct within an interaction trajectory, while reward decay calibrates credit across the batch by assigning smaller rewards to answers that require more rounds of teacher guidance.}
\vspace{-5pt}
\label{fig:credit-assignment}
\end{figure*}

Step-wise rewards alone, however, do not solve credit assignment across samples in the same batch. If every correct response receives the same reward, then a problem solved independently at step $1$ and another problem solved only after several rounds of guidance would contribute the same positive reward, even though they reflect different degrees of student competence. This creates an assisted reward-hacking failure mode: the student can benefit from teacher intervention during training without learning to solve earlier or independently. We discuss this situation in Section~\ref{sec:reward_hacking_analysis}. Therefore, after locating credit within each trajectory, SocraticPO further calibrates credit across the batch by decaying rewards for later-step correctness.

Consider a training batch. Let $\mathcal A_k$ be the set of samples that are active at step $k$, meaning they have not been solved before step $k$. For sample $i\in\mathcal A_k$, let
\begin{equation}
\delta_i^{(k)}=
\mathbb{I}[\text{correct}(y_i^{(k)},\mathcal G_i)].
\end{equation}
The first-step reward is
\begin{equation}
r_i^{(1)}=\delta_i^{(1)}.
\end{equation}
For $k\ge 2$, we compute the historical average reward over previous active attempts:
\begin{equation}
\mu_{k-1}
=
\frac{
\sum_{j=1}^{k-1}\sum_{i\in\mathcal A_j}r_i^{(j)}
}{
\sum_{j=1}^{k-1}|\mathcal A_j|
},
\end{equation}
and define
\begin{equation}
r_i^{(k)}
=
\mu_{k-1}\cdot \delta_i^{(k)},\qquad i\in\mathcal A_k.
\end{equation}
Incorrect responses always receive zero reward. Correct responses at later steps receive a reward scaled by the average success value of earlier attempts, which makes delayed correction less valuable than immediate success. This reward rule gives the following properties and the proof is provided in Appendix~\ref{app:proof_advantage}:
\begin{algorithm}[t]
\caption{Socratic Policy Optimization with Conditional Teacher Intervention}
\label{alg:spo}
\KwIn{Student policy $\pi_\theta$; teacher policy $\pi_\phi$; prompt dataset $\mathcal D$; verifier information $\mathcal G$; teacher-side information $\mathcal I$; maximum interaction steps $K$; batch size $B$.}
\KwOut{Optimized student policy $\pi_\theta$.}
\textbf{repeat}
Sample a batch $\{(x_i,\mathcal G_i,\mathcal I_i)\}_{i=1}^{B}\sim\mathcal D$\;
Initialize $c_i^{(1)}\gets x_i$, active set $\mathcal A_1\gets\{1,\dots,B\}$, and trajectory buffer $\mathcal T\gets\emptyset$\;
\For{$k=1$ \KwTo $K$}{
    \ForEach{$i\in\mathcal A_k$}{
        Generate $y_i^{(k)}\sim\pi_\theta(\cdot\mid c_i^{(k)})$;\quad
        Compute $\delta_i^{(k)}\gets\mathbb{I}[\text{correct}(y_i^{(k)},\mathcal G_i)]$;\quad
        Store $(c_i^{(k)},y_i^{(k)},\delta_i^{(k)})$ in $\mathcal T$\;
        \If{$\delta_i^{(k)}=0$ \textbf{and} $k<K$}{
            Generate $f_i^{(k)}\sim\pi_\phi(\cdot\mid c_i^{(k)},y_i^{(k)},\mathcal I_i)$;\quad Set $c_i^{(k+1)}\gets c_i^{(k)} || y_i^{(k)} || f_i^{(k)}$\;
        }
    }
    Compute rewards $\{r_i^{(k)}: i\in\mathcal A_k\}$ using the step-wise rule\;
    Set $\mathcal A_{k+1}\gets\{i\in\mathcal A_k:\delta_i^{(k)}=0\}$\;
    \If{$\mathcal A_{k+1}=\emptyset$}{
        \textbf{break}\;
    }
}
Update $\pi_\theta$ with Reinforce++ using student responses in $\mathcal T$ and their rewards\;
\textbf{until convergence}
\end{algorithm}

\begin{theorem}[Assistance-Aware Advantage under Reward Decay]
\label{thm:assistance_advantage}
Define the normalized advantage
\begin{equation}
\hat a_i^{(k)}=\frac{r_i^{(k)}-\mu_t}{\sigma_t+\epsilon},
\end{equation}
where $\mu_t,\sigma_t$ are the mean and standard deviation after step $t$ with $1\le t\le k\le K$. Then the following properties hold:
\begin{enumerate}[leftmargin=*]
    \item \textbf{(Non-negativity)} If sample $i$ is correct at step $k$, i.e., $\delta_i^{(k)}=1$, then $\hat a_i^{(k)}\ge 0$.
    \item \textbf{(Monotonicity)} If $t\geq 2$ and samples $i,j$ are correct at step $k,l(k<l)$, i.e., $\delta_i^{(k)}=\delta_j^{(l)}=1$, then $\hat a_i^{(k)}>\hat{a}_j^{(l)}$.
    \item \textbf{(Saturation decay)} Let
$
p_k=\frac{1}{|\mathcal A_k|}\sum_{i\in\mathcal A_k}\delta_i^{(k)}
$
be the correction rate among active samples. Then if $p_t\rightarrow 1$, we have $\hat{a}_i^{(t)}\rightarrow 0$.
\end{enumerate}
\end{theorem}

The first property shows that SocraticPO does not punish corrected reasoning. Once a sample is answered correctly, the corresponding student tokens receive a non-negative normalized advantage, even if the answer is obtained after teacher guidance. Thus, \textbf{reward decay does not discard assisted successes}; it keeps them as valid positive learning signals.

The second property formalizes the batch-level credit assignment discussed above. Among samples that eventually become correct within the same batch, SocraticPO gives a larger normalized advantage to those corrected at earlier steps. This prevents an independently solved problem and a problem solved only after multiple rounds of guidance from receiving indistinguishable credit. In this sense, \textbf{reward decay calibrates positive reward across samples by favoring trajectories that become correct with fewer guidance rounds}.

The third property explains how reward decay mitigates assisted reward hacking. If teacher guidance makes almost all active samples correct, then $p_k\rightarrow 1$ and the advantage approaches zero. Once a guidance round becomes non-discriminative because nearly everyone can solve the task with that amount of help, trajectories from that round no longer provide strong positive reinforcement. This means that \textbf{reward decay can prevent the student policy from gaining excessive advantage merely by relying on teacher intervention}.


\section{Experiments}
\label{sec:experiments}

We evaluate SocraticPO on scientific reasoning tasks. The experiments are designed to answer two main questions: (1) whether SocraticPO improves over standard RL and student--teacher baselines; and (2) whether its two key components, teacher guidance and reward decay, are both necessary for effective learning. We further include diagnostic analyses on teacher-side information and teacher--student similarity to better understand when the interaction mechanism helps.

\subsection{Experimental Setup}

\paragraph{Models.}
We use \textbf{Qwen3-4B-Instruct} as the student policy $\pi_\theta$. For teacher guidance, we consider two settings: \textbf{Qwen3-4B-Instruct} as a self-guidance teacher and \textbf{Qwen3.5-27B} as a stronger teacher. This allows us to separate the effect of the interaction mechanism from the effect of teacher capacity.

\paragraph{Datasets and metric.}
Following the SDPO evaluation setup, we use four undergraduate-level scientific reasoning benchmarks from the reasoning subsets of \textbf{SciKnowEval}~\cite{feng2024sciknoweval}: \textbf{Chemistry}, \textbf{Physics}, \textbf{Biology}, and \textbf{Material}. We report accuracy on each domain and the unweighted average across domains.

\begin{table}[t]
  \caption{Accuracy on SciKnowEval reasoning benchmarks. SocraticPO variants specify the teacher model in parentheses. Bold indicates the best result in each column, and underline indicates the second-best result.}
  \label{tab:main_results}
  \centering
  \vspace{5pt}
  \resizebox{0.86\textwidth}{!}{
  \begin{tabular}{lcccccc}
    \toprule
    Method  & Chemistry & Physics & Biology & Material & Avg. \\
    \midrule
    Qwen3-4B-Instruct  & 0.5071 & 0.6664 & 0.2762 & 0.7413 & 0.5478 \\
    +Reinforce++  & 0.7270 & 0.7726 & 0.6575 & 0.7978 & 0.7388 \\
    +GRPO  & 0.7657 & \underline{0.7773} & 0.6137 & 0.7712 & 0.7320 \\
    +GSPO  & 0.7633 & 0.7697 & 0.5875 & 0.7849 & 0.7262 \\
    +SDPO  & \textbf{0.7961} & 0.7742 & 0.5987 & 0.7779 & 0.7367 \\
    \midrule
    +SocraticPO(Qwen3-4B) & 0.7672 & 0.7648 & \underline{0.6612} & \underline{0.8058} & \underline{0.7498} \\
    +SocraticPO(Qwen3.5-27B) & \underline{0.7782} & \textbf{0.7789} & \textbf{0.6787} & \textbf{0.8138} & \textbf{0.7624} \\
    \bottomrule
  \end{tabular}}
\end{table}
\begin{figure}[t]
    \centering
    \includegraphics[width=0.8\linewidth]{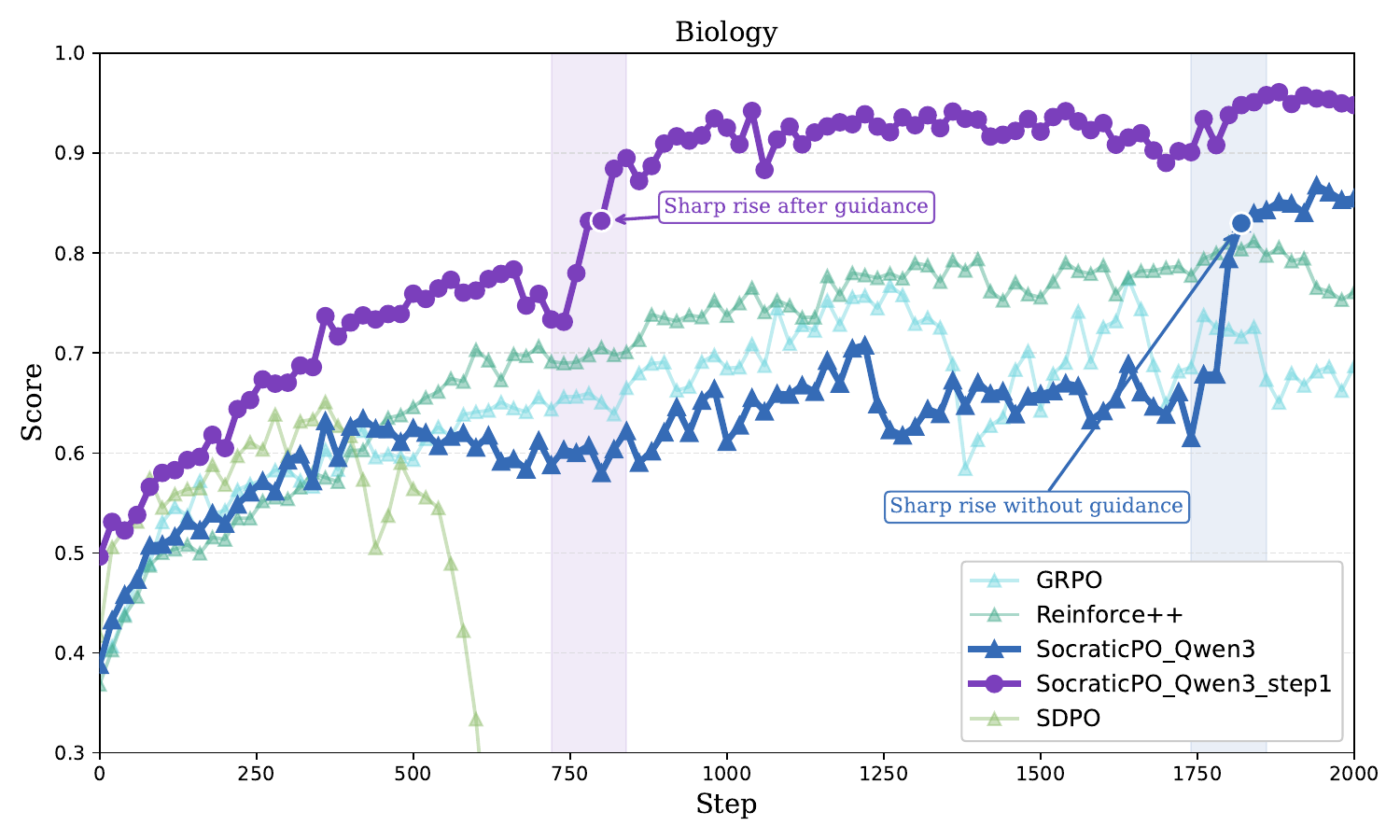}
    \caption{Training dynamics of SocraticPO compared with baseline policy optimization. The curve illustrates how interaction-conditioned rollouts affect performance over training steps.}
    \label{fig:train_step}
\end{figure}
\paragraph{Baselines.}
We compare against \textbf{Reinforce++}~\cite{hu2025reinforce}, \textbf{GRPO}~\cite{shao2024deepseekmath}, \textbf{GSPO}~\cite{zheng2025group}, and \textbf{SDPO}~\cite{hubotter2026reinforcement}. Reinforce++, GRPO and GSPO represent RL methods driven by scalar rewards, while SDPO represents a self-distillation-style baseline.

\subsection{Main Results}
Table~\ref{tab:main_results} reports the main results. With Qwen3.5-27B as the teacher, SocraticPO achieves the best average accuracy, improving from 0.7388 with Reinforce++ to 0.7624. It also obtains the best score on Physics, Biology, and Material, suggesting that Socratic guidance is especially useful when errors can be localized and corrected through intermediate reasoning feedback. Chemistry is the main exception, where SDPO achieves the highest single-domain score. Under self-guidance, SocraticPO remains competitive and reaches an average accuracy of 0.7498, which is still higher than Reinforce++ and GRPO. This indicates that the interaction mechanism itself contributes to learning, while a stronger teacher further amplifies the benefit.

In Figure~\ref{fig:train_step}, we observe a two-stage pattern in the training dynamics. In the early stage, the student rapidly improves its ability to answer correctly after receiving teacher guidance, as reflected by the guided-performance curve. This is consistent with the role of interaction-conditioned rollouts: the model sees its own failed attempts followed by targeted corrective feedback, and first learns how to use guidance to repair its reasoning. In the later stage, reward decay gradually reduces the advantage of answers that are correct only after guidance, while unassisted correct answers receive relatively stronger credit. As a result, the learning pressure shifts from ``being correct with help'' to ``being correct earlier and without help.'' Combined with the correction ability learned in the early stage, this shift eventually improves the student's unassisted performance as well. This pattern supports our design intuition: teacher guidance creates learnable repairs, and reward decay encourages the student to internalize those repairs into independent reasoning.

\subsection{Ablation Study}
\label{sec:ablation}
\label{sec:reward_hacking_analysis}

\begin{table}[t]
  \caption{Ablation results for teacher guidance and reward decay.}
  \label{tab:ablation_results}
  \centering
  \vspace{5pt}
  \resizebox{0.8\textwidth}{!}{
  \begin{tabular}{ccccccc}
    \toprule
     Teacher & Method & Chemistry & Physics & Biology & Material & Avg. \\

    \midrule
    -- & w/o tg & 0.7110 & 0.7703 & 0.5737 & 0.7706 & 0.7064 \\
    \midrule
    \multirow{3}{*}{Qwen3-4B} & w/o rd & 0.7586 & 0.7468 & 0.5837 & 0.7859 & 0.7188 \\
    & 0.8 rd  & 0.7661 & 0.7753& 0.6400  & 0.7812 & 0.7406 \\
     & --  & 0.7672 & 0.7648 & 0.6612 & 0.8058 & 0.7498 \\
     \midrule
    \multirow{3}{*}{Qwen3.5-27B} & w/o rd & 0.7053 & 0.7484 & 0.5700 & 0.7992 & 0.7057 \\
    & 0.8 rd  & 0.7487 & 0.7695 & 0.5909 & 0.7945 & 0.7259 \\
     & -- & \textbf{0.7782} & \textbf{0.7789} & \textbf{0.6787} & \textbf{0.8138} & \textbf{0.7624} \\
    \bottomrule
  \end{tabular}}
\end{table}

\begin{wrapfigure}{r}{0.46\textwidth}
    \vspace{-2.8em}
    \centering
    \includegraphics[width=\linewidth]{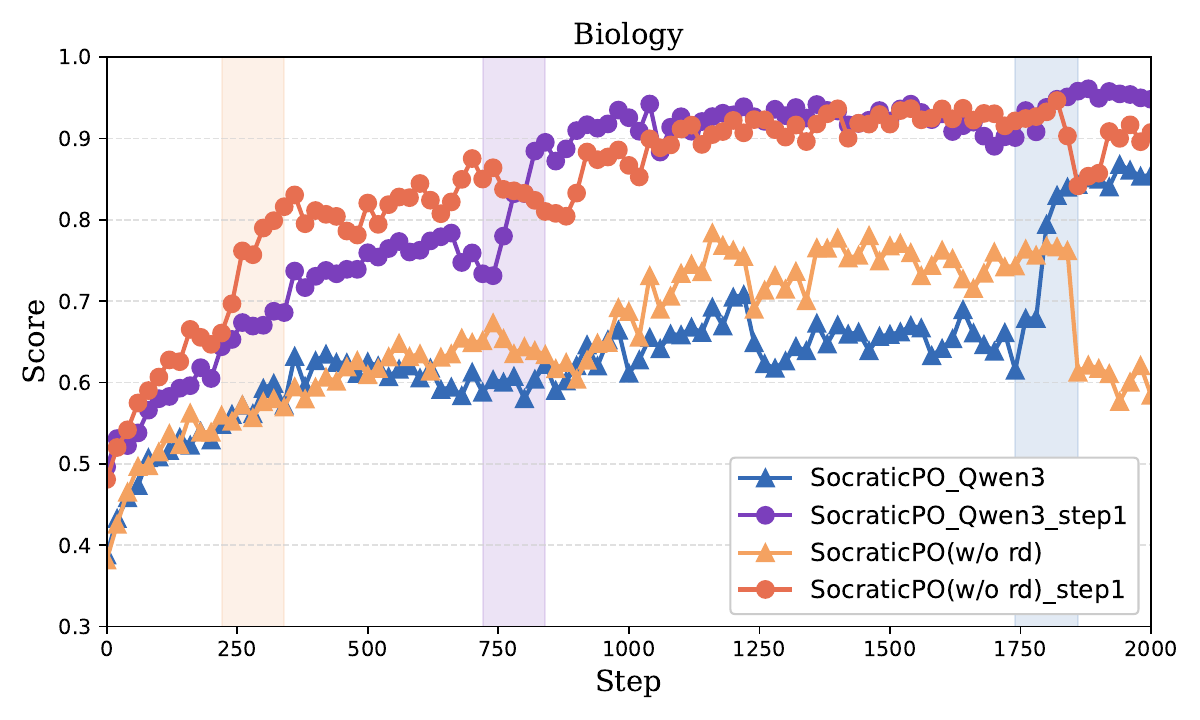}
    \vspace{-1.8em}
    \caption{Ablation comparison for teacher guidance and reward decay.}
    \label{fig:ablation}
    \vspace{-1.8em}
\end{wrapfigure}
We evaluate three variants of SocraticPO. \textbf{SocraticPO (w/o tg)} removes adaptive teacher guidance and replaces it with a fixed retry prompt. \textbf{SocraticPO (w/o rd)} keeps teacher guidance but removes reward decay, assigning full reward to any correct answer regardless of the interaction step. \textbf{SocraticPO (0.8 rd)} replaces our step-dependent reward decay with a fixed exponential decay factor of 0.8.

Table~\ref{tab:ablation_results} separates the component analysis from the main comparison. Removing teacher guidance and using a fixed retry prompt underperforms full SocraticPO, especially on Biology and Material. This shows that the improvement is not explained by repeated attempts alone: the teacher must provide feedback conditioned on the student's actual mistake so that the next rollout context contains useful correction.

Reward decay is even more critical. As visualized in Figure~\ref{fig:ablation}, removing reward decay can make the guided ability improve faster in the early stage, as shown by the orange curve. This is expected: once corrected answers after guidance receive the same reward as independent correct answers, the model is strongly incentivized to exploit teacher-assisted contexts. However, this early improvement does not translate effectively into unassisted ability in the later stage. The resulting performance drop is substantial, especially with the stronger teacher, where the average decreases from 0.7624 to 0.7057. \textbf{Using a fixed decay factor of 0.8 partially recovers the performance to 0.7406 and 0.7259 under the two teachers, but still underperforms the full reward-decay scheme.} This supports the assisted reward-hacking concern: teacher guidance can quickly improve assisted correctness, but reward decay is needed to turn those guided repairs into independent reasoning ability.

\subsection{Additional Analysis}

\paragraph{Teacher access to ground truth.}
We study whether the teacher should receive ground-truth answers. Although such access seems useful for generating precise feedback, Figure~\ref{fig:ground_truth} shows that it can degrade student performance in most cases. We attribute this to two factors. First, even when the prompt explicitly asks the teacher not to reveal the answer, a teacher with ground-truth access may still leak solution-specific hints. This failure mode is inherently limited by the teacher model's instruction-following ability. Second, ground-truth access can shift the teacher from diagnosing the student's reasoning process to reverse-engineering hints from the known answer. Both effects may encourage shortcut learning rather than robust correction.

At the same time, ground-truth information is not intrinsically harmful. If the teacher can avoid answer leakage and use the reference solution to diagnose the student's error, then access to the answer or even a standard reasoning process can substantially improve the quality of guidance, as suggested by the Qwen3-8B-Instruct self-guidance result on Chemistry. The effect of ground-truth access is therefore conditional: \textbf{it can strengthen Socratic guidance when the teacher follows the teaching protocol well, but it can become negative when the teacher turns the answer into overly direct hints}.
\begin{figure}[t]
    \centering
    \includegraphics[width=0.8\linewidth]{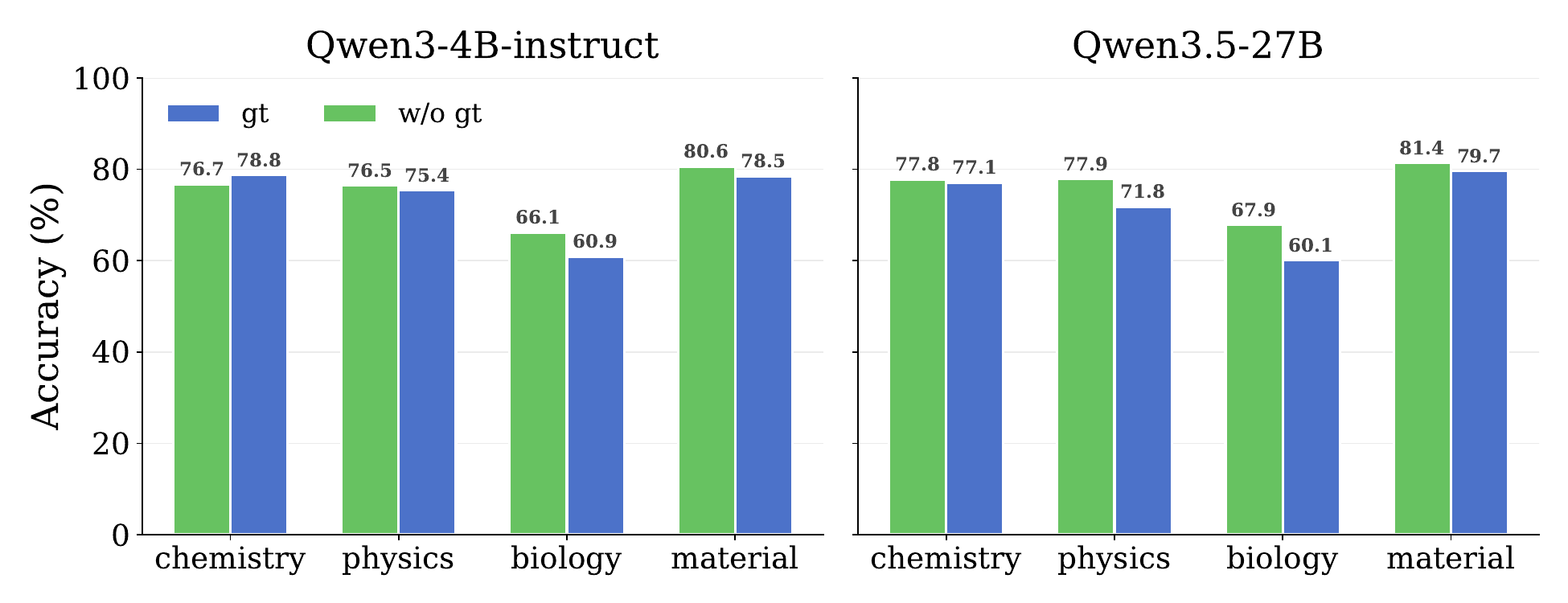}
    \caption{Effect of teacher access to ground-truth information.}
    \vspace{-3pt}
    \label{fig:ground_truth}
\end{figure}

\begin{wrapfigure}{r}{0.4\textwidth}
    \vspace{-2.8em}
    \centering
    \includegraphics[width=\linewidth]{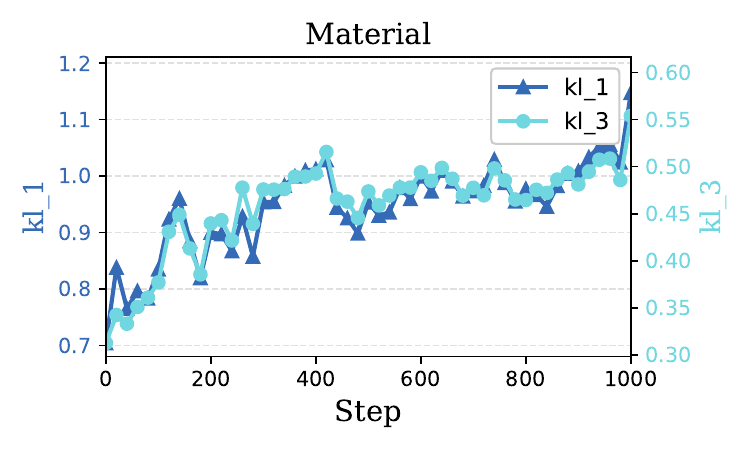}
    \vspace{-2.8em}
    \caption{Teacher--student KL divergence under SocraticPO on Material.}
    \label{fig:teacher_student_kl}
    \vspace{-1.8em}
\end{wrapfigure}
\paragraph{Teacher--student similarity.}
To examine whether SocraticPO works by implicitly making the student imitate the teacher, we measure the KL divergence between the student and teacher during training. We use \textbf{Qwen3-8B} as the teacher and \textbf{Qwen3-4B-Instruct} as the student on the Material benchmark. This pair is chosen because KL computation requires aligned tokenization, and the two models share a compatible tokenizer.

If SocraticPO were primarily a distillation mechanism, we would expect teacher--student divergence to decrease as training proceeds. Figure~\ref{fig:teacher_student_kl} shows that SocraticPO does not consistently reduce the KL divergence between teacher and student, even though it improves downstream task performance. This supports the distinction made in our method: the teacher participates in rollout as an interactive guide, but the student is not optimized to reproduce the teacher's token distribution or reasoning trace. In other words, SocraticPO improves the student through guided correction rather than by making the student become more teacher-like.

Together, these analyses suggest that effective teaching protocols require more than a stronger teacher. \textbf{The information exposed to the teacher, the form of feedback, and the training objective jointly determine whether interaction improves policy learning.}

\section{Conclusion}
\label{sec:conclusion}

We presented \textbf{SocraticPO}, a policy-optimization framework that addresses a limitation of scalar reward feedback in LLM reinforcement learning. Outcome rewards indicate whether a response should be reinforced, but they do not explain how a mistaken reasoning process should be repaired. SocraticPO augments rollout sampling with natural-language teacher guidance: the student first attempts a solution, receives targeted correction only after an incorrect response, and continues under the expanded dialogue context.

Teacher guidance must be controlled carefully. If assisted correct answers are rewarded as strongly as independent correct answers, the student may exploit teacher intervention instead of improving its own reasoning. SocraticPO therefore pairs guidance with reward decay, assigning smaller rewards to answers that become correct only after more interaction rounds. Experiments on SciKnowEval scientific reasoning benchmarks show that SocraticPO improves average performance over strong RL and self-distillation baselines, while ablations highlight the importance of both targeted guidance and reward decay.

Our findings suggest that natural-language correction can serve as process-level supervision inside RL without requiring teacher distribution matching. Because the teacher is used only as a text-producing guide, SocraticPO can leverage black-box mentors and remains modular with respect to the underlying RL optimizer. Future work can further study teaching protocols, teacher-side information, and broader tasks where reasoning errors are diagnosable before the final answer.

\bibliographystyle{plainnat}
\bibliography{main}
\newpage
\appendix
\section{Proof of Theorem 1}
\label{app:proof_advantage}
\begin{proof}
We prove the three properties in Theorem~\ref{thm:assistance_advantage}. Let $m_k=|\mathcal A_k|$ and
\begin{equation}
p_k=\frac{1}{m_k}\sum_{i\in\mathcal A_k}\delta_i^{(k)}.
\end{equation}
Let $N_{k-1}=\sum_{j=1}^{k-1}|\mathcal A_j|$. The batch mean after including the step-$k$ active samples is
\begin{equation}
\mu_k=
\frac{N_{k-1}\mu_{k-1}+\sum_{i\in\mathcal A_k}r_i^{(k)}}{N_{k-1}+m_k}.
\end{equation}
By the reward definition,
\begin{equation}
\sum_{i\in\mathcal A_k}r_i^{(k)}
=
\mu_{k-1}\sum_{i\in\mathcal A_k}\delta_i^{(k)}
=
\mu_{k-1}m_kp_k.
\end{equation}
Thus,
\begin{equation}
\mu_k
=
\frac{N_{k-1}\mu_{k-1}+m_kp_k\mu_{k-1}}{N_{k-1}+m_k}
=
\mu_{k-1}
\frac{N_{k-1}+m_kp_k}{N_{k-1}+m_k}.
\end{equation}
For any correct sample $i$ at step $k$, we have $\delta_i^{(k)}=1$ and therefore $r_i^{(k)}=\mu_{k-1}$. Its centered advantage is
\begin{equation}
a_i^{(k)}
=
r_i^{(k)}-\mu_t
=
\mu_{k-1}-\mu_t
=
\mu_{k-1}
\left(
1-\frac{N_{k-1}+\sum_{h=k}^{t}m_kp_k}{N_{k-1}+\sum_{h=k}^{t}m_k}
\right)
\geq 0.
\end{equation}
The normalized advantage has the same sign because $\sigma_t+\epsilon>0$, so $\hat a_i^{(k)}\ge 0$. This proves the first property.

For the second property, consider two correct samples at steps $k$ and $l$ with $k<l$. 
\begin{equation}
    a_i^{(k)}-a_j^{(l)}=r_i^{(k)}-\mu_t-(r_j^{(l)}-\mu_t)=\mu_{k-1}-\mu_{l-1}=\mu_{k-1}\left(1-\frac{N_{k-1}+\sum_{h=k}^{l-1}m_hp_h}{N_{k-1}+\sum_{h=k}^{l-1}m_h}\right).
\end{equation}
Consider that $l>k$, according to the definition of rollout, it is inevitable that $p_k<1$, then $a_i^{(k)}-a_j^{(l)}>0$, which lead to $\hat a_i^{(k)}\ge \hat{a}_j^{(l)}$. 

For the third property, the expression above shows that the centered advantage for a correct sample at step $t$ is proportional to $1-p_t$:
\begin{equation}
a_i^{(t)}  = 
\frac{\mu_{t-1}m_t}{N_{t-1}+m_t}(1-p_t)
\end{equation}
Therefore, when $p_k\rightarrow 1$, i.e., almost all active samples are corrected at that step, we have $a_i^{(k)}\rightarrow 0$. Since $\sigma_t+\epsilon>\epsilon>0$,  the normalized advantage also satisfies $\hat a_i^{(t)}\rightarrow 0$.
\end{proof}
\section{Detailed Experimental Settings}
\label{app:experimental_settings}

\subsection{Example Interaction Trace}
\label{app:interaction_trace}

The following trace shows one SocraticPO interaction from the Material domain. Ellipses indicate text omitted for brevity while preserving the original dialogue structure.

\begin{quote}
\small
\noindent\rule{\linewidth}{0.4pt}
\textbf{Task instruction.} Given a question and four options, select the right answer. Respond with
\texttt{<reasoning> ... </reasoning>} and \texttt{<answer> ... </answer>}. For the answer, only output the letter corresponding to the correct option.

\vspace{0.4em}
\noindent\textbf{Question.} Based on the given data, determine which material is most favorable for the stability of perovskite materials. The stability can be predicted by element properties such as density and first ionization potential.

\vspace{0.2em}
\noindent\textbf{Given data.}

Y4Sr4Fe2Co6O24: Density 6.07425, FIP 6.95;

Y4Sr4Fe2Ni6O24: Density 6.07425, FIP 6.865625;

Y4Sr4Mn1Co7O24: Density 6.11125, FIP 6.9221875;

Y4Sr4Mn1Ni7O24: Density 7.61475, FIP 8.3265625.

\vspace{0.2em}
\noindent\textbf{Options.}

A: Y4Sr4Mn1Co7O24;

B: Y4Sr4Mn1Ni7O24;

C: Y4Sr4Fe2Co6O24;

D: Y4Sr4Fe2Ni6O24.

\vspace{0.4em}
\noindent\textbf{Student attempt 1.}
\texttt{<reasoning>}
To determine which perovskite material is most favorable for stability, we are given two key properties: density and first ionization potential.
... Therefore, the best choice is C.
\texttt{</reasoning>}
\texttt{<answer>C</answer>}

\vspace{0.4em}
\noindent\textbf{Teacher feedback 1.}
Your reasoning contains a critical contradiction in Step 2. You stated that ``Material C also has a slightly higher density (6.111) than Material D,'' but the data table shows Material C has density 6.07425, while Material A is the one with 6.11125. ...
\textit{Next Step}: (1) Re-check the density values for each material. (2) Re-evaluate the trend: ...

\vspace{0.4em}
\noindent\textbf{Student attempt 2.}
\texttt{<reasoning>}
I apologize for the earlier error. Let's carefully re-examine the data and correct the reasoning step by step.
... While high density alone is not a guarantee, in combination with high FIP, it suggests a stable, tightly bound structure.
Final decision: Material B has the best combination of high ionization potential and high density.
Therefore, the most favorable for stability is B.
\texttt{</reasoning>}
\texttt{<answer>B</answer>}

\vspace{0.4em}
\noindent\textbf{Verifier outcome.} Attempt 1 is incorrect; attempt 2 is correct. The trajectory is solved at interaction step 2.
\noindent\rule{\linewidth}{0.4pt}
\end{quote}

\subsection{Teacher Prompt Templates}
\label{app:prompt_templates}
\begin{quote}
\small
\noindent\rule{\linewidth}{0.4pt}
You are an expert teacher. Student's answer is incorrect. Please review the student's latest answer and provide concise, actionable guidance. Prefer hints, corrections, and next-step suggestions over directly revealing the answer.

\textbf{TEACHING PRINCIPLE:}
\begin{itemize}
    \item Point out the main issue in the student's answer.
    \item Give the next step or correction they should try.
    \item Do not solve the full problem for them unless explicitly allowed.
    \item Keep the guidance short and concrete.
\end{itemize}
\textbf{Question:} \{question\} \newline
\textbf{Student's latest answer:} \{student\_answer\} \newline
\textbf{Reference solution:} \{reference\_solution\} \newline
\textbf{Ground-truth block:} \{ground\_truth\_block\} \newline
\textbf{Output Format:} Write a short piece of guidance for the student.\newline
\noindent\rule{\linewidth}{0.4pt}
\end{quote}
\newpage
\subsection{Dataset Statistics}
\label{app:dataset_statistics}
We use the reasoning subsets of SciKnowEval in four scientific domains. Table~\ref{tab:data_stats} reports the number of training and evaluation examples in each domain.
\begin{table}[h]
\centering
\caption{Sciknoweval dataset statistics.}
\label{tab:data_stats}
\resizebox{0.6\textwidth}{!}{
\begin{tabular}{lccc}
\toprule
Domain & Train examples & Evaluation examples & Total \\
\midrule
Chemistry & 1,890 & 210 & 2100 \\
Physics & 720 & 80 & 800 \\
Biology & 450 & 50 & 500 \\
Material & 841 & 94 & 935 \\
\bottomrule
\end{tabular}}
\end{table}

Table~\ref{tab:benchmark_examples} shows representative examples from the four benchmark domains. Each task is formatted as a multiple-choice question, and the model is evaluated by the final option letter.

\begin{table}[h]
\centering
\caption{Representative examples from the SciKnowEval reasoning subsets.}
\label{tab:benchmark_examples}
\small
\begin{tabular}{p{0.12\textwidth}p{0.72\textwidth}p{0.08\textwidth}}
\toprule
Benchmark & Question (excerpt) & Answer \\
\midrule
Chemistry &
Which of the following molecules is the only correct reactant used in the synthesis of the product "Nc1ccccc1S(=O)(=O)N1CC=CC1" ?

A: O=[N+]([O-])c1ccccc1S(=O)(=O)N1CC=CC1 

B: O=[N+]([O-])c1ccccc1S(=O)(=O)NC1CCCCC1 

C: O=[N+]([O-])c1ccccc1S(=O)(=O)Cl 

D: O=[N+]([O-])c1ccccc1S(=O)(=O)NC1CCC1
& B \\
\midrule
Physics &
What is the rotational period of the Milky Way's galactic halo, given its average radius of 77122 light-years? Please note that the following content involves constants might be used in the calculations :pi = 3.14 ; velocity of light: c = 3e8 m/s ;light year = 9.461e15 m. Additionally, please select the option from the given choices that you believe is closest to the correct answer!

A: 6.61e+08 years \quad B: 4.97e+08 years \quad C: 1.7e+08 years    \quad D: 4.1e+08 years
& B \\
\midrule
Biology &
With a fitness of 1.0 for the wide-type protein sequence 'MQYKLILNGKTLKGETTTEAVDAATAEKVFKQYANDNGVDGEWTYDDATKTFTVTE', a four-sites random mutation has been carried out to produce the mutant 'MQYKLILNGKTLKGETTTEAVDAATAEKVFKQYANDNGQWMEWTYDDATKTFTHTE'. What is the new fitness value of this mutant?

A: $0.01$ \quad B: $0.022$ \quad C: $0.0$ \quad D: $2.774$
& B \\
\midrule
Materials &
Given the following crystal structure parameters for the material RbLa9(IrO6)4 (Material ID: mp-560657), calculate the volume of the unit cell (in \AA$^3$). Lattice: $a=7.82$, $b=7.82$, $c=17.88$ \AA, $\alpha=\beta=\gamma=90^\circ$.

A: 1025.67 \quad B: 1094.31 \quad C: 1200.45 \quad D: 1150.78
& B \\
\bottomrule
\end{tabular}
\end{table}
\newpage
\subsection{Implementation Details}
\label{app:implementation_details}

Table~\ref{tab:hyperparams} summarizes the main training and inference settings in a grouped format.
All experiments are conducted on a 6-card H100 GPU machine. We allocate 4 GPUs for student-policy training and 2 GPUs for teacher inference during SocraticPO rollouts, so that teacher feedback can be generated online without blocking the student update pipeline as much as possible.

\begin{table}[h]
\centering
\caption{Hyperparameters for Reinforce++, GRPO, SDPO, and SocraticPO. Entries marked ``--'' indicate parameters not applicable to that method.}
\label{tab:hyperparams}
\resizebox{\textwidth}{!}{
\begin{tabular}{lcccc}
\toprule
Parameters & Reinforce++ & GRPO & SDPO & SocraticPO \\
\midrule
\multicolumn{5}{l}{\textbf{General}} \\
Student model & Qwen3-4B-Instruct & Qwen3-4B-Instruct & Qwen3-4B-Instruct & Qwen3-4B-Instruct \\
Teacher model & -- & -- & Qwen3-4B-Instruct & Qwen\{3-4B-Instruct,3.5-27B\} \\
Thinking mode & False & False & False & False \\
\midrule
\multicolumn{5}{l}{\textbf{Data}} \\
Max. prompt length & 8192 & 8192 & 8192 & 8192 \\
Max. response length & 8192 & 8192 & 8192 & 8192 \\
Max. teacher feedback length & -- & -- & -- & 1024 \\
\midrule
\multicolumn{5}{l}{\textbf{Batching}} \\
Question batch size & 32 & 32 & 32 & 32 \\
Mini batch size & 8 & 8 & 8 & 8 \\
Number of rollouts & 8 & 8 & 8 & 8 \\
Maximum interaction steps $K$ & -- & -- & -- & 2 \\
\midrule
\multicolumn{5}{l}{\textbf{Rollout}} \\
Inference engine & VLLM & VLLM & VLLM & VLLM \\
Student temperature & 1.0 & 1.0 & 1.0 & 1.0 \\
Teacher temperature & -- & -- & 1.0 & 0.2 \\
\midrule
\multicolumn{5}{l}{\textbf{Validation}} \\
Number of rollouts & 16 & 16 & 16 & 16 \\
Temperature & 0.6 & 0.6 & 0.6 & 0.6 \\
Top-$p$ & 0.95 & 0.95 & 0.95 & 0.95 \\
\midrule
\multicolumn{5}{l}{\textbf{RL loss}} \\
$\epsilon$-high asymmetric clip & -- & 0.28 & -- & 0.28 \\
Rollout IS clip $(\rho)$ & -- & 2 & -- & 2 \\
KL coefficient & 0.0 & 0.0 & -- & 0.0 \\
\midrule
\multicolumn{5}{l}{\textbf{Training}} \\
Optimizer & AdamW & AdamW & AdamW & AdamW \\
Learning rate & $1\times10^{-6}$ & $1\times10^{-6}$ & $1\times10^{-5}$ & $1\times10^{-6}$ \\
Warmup steps & 10 & 10 & 10 & 10 \\
Weight decay & 0.01 & 0.01 & 0.01 & 0.01 \\
Gradient clip norm & 1.0 & 1.0 & 1.0 & 1.0 \\
\bottomrule
\end{tabular}}
\end{table}

\end{document}